\documentclass[10pt,twocolumn,letterpaper]{article}

\usepackage{iccv}
\usepackage{times}
\usepackage{epsfig}
\usepackage{graphicx}
\usepackage{amsmath}
\usepackage{amssymb}

\graphicspath{ {./imgs/} }

\usepackage{amsmath,amsfonts,bm}

\def\eqref#1{equation~\ref{#1}}

\def\1{\bm{1}}

\def\vtheta{{\bm{\theta}}}

\def\vx{{\bm{x}}}

\DeclareMathAlphabet{\mathsfit}{\encodingdefault}{\sfdefault}{m}{sl}
\SetMathAlphabet{\mathsfit}{bold}{\encodingdefault}{\sfdefault}{bx}{n}

\usepackage{amsmath}
\usepackage{multirow}
\usepackage{booktabs}
\usepackage{algorithm}
\usepackage{algpseudocode}
\usepackage{amsthm}
\theoremstyle{definition}
\newtheorem{definition}{Definition}[section]
\usepackage[dvipsnames]{xcolor}
\usepackage{xspace}

\newtheorem{theorem}{Theorem}[section]

\newcommand{\Break}{\State \textbf{break} }

\definecolor{mypink}{HTML}{FD8A8A}
\definecolor{mypurple}{HTML}{A459D1}

\newcommand{\method}{RiFT\xspace}
\usepackage{todonotes}
\usepackage{xcolor}

\newcommand{\eq}[1]{Eq.~(\ref{#1})}

\usepackage[pagebackref=true,breaklinks=true, letterpaper=true,colorlinks,bookmarks=false]{hyperref}

\iccvfinalcopy 

\def\iccvPaperID{10616} 

\ificcvfinal\fi

\begin{document}
\title{Improving Generalization of Adversarial Training via \\Robust Critical Fine-Tuning}

\author{ \textbf{Kaijie Zhu}$^{1,2}$, \textbf{Jindong Wang}$^3$, \textbf{Xixu Hu}$^4$, \textbf{Xing Xie}$^3$, \textbf{Ge Yang}$^{1,2}$ \thanks{Corresponding author: Ge Yang $<$ge.yang@ia.ac.cn$>$.} \\
 $^1$School of Artifical Intelligence, University of Chinese Academy of Sciences ~~ \\ $^2$Institute of Automation, CAS ~~ $^3$ Microsoft Research ~~ $^4$ City University of Hong Kong
}

\maketitle
\ificcvfinal\thispagestyle{empty}\fi



\begin{abstract}
Deep neural networks are susceptible to adversarial examples, posing a significant security risk in critical applications. Adversarial Training (AT) is a well-established technique to enhance adversarial robustness, but it often comes at the cost of decreased generalization ability. This paper proposes Robustness Critical Fine-Tuning (\textbf{\method}), a novel approach to enhance generalization \emph{without} compromising adversarial robustness. The core idea of \method is to exploit the redundant capacity for robustness by fine-tuning the adversarially trained model on its non-robust-critical module. To do so, we introduce \emph{module robust criticality (MRC)}, a measure that evaluates the significance of a given module to model robustness under worst-case weight perturbations. Using this measure, we identify the module with the lowest MRC value as the non-robust-critical module and fine-tune its weights to obtain fine-tuned weights. Subsequently, we linearly interpolate between the adversarially trained weights and fine-tuned weights to derive the optimal fine-tuned model weights. We demonstrate the efficacy of \method on ResNet18, ResNet34, and WideResNet34-10 models trained on CIFAR10, CIFAR100, and Tiny-ImageNet datasets. Our experiments show that \method can significantly improve both generalization and out-of-distribution robustness by around $1.5$\% while maintaining or even slightly enhancing adversarial robustness. Code is available at \url{https://github.com/microsoft/robustlearn}.
\end{abstract}



\section{Introduction}

\begin{figure}[ht]
\centering
\includegraphics[width=0.5\textwidth]{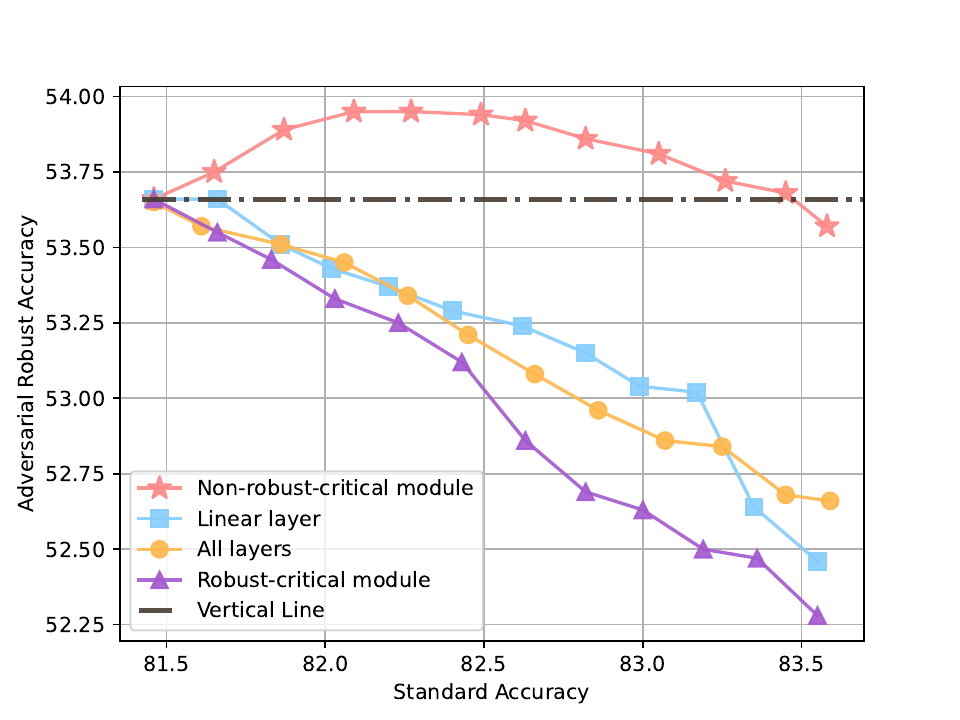}
\caption{Interpolation results of fine-tuning on different modules of ResNet18 on CIFAR10 dataset. 
Dots denote different interpolation points between the final fine-tuned weights of \method and the initial adversarially trained weights. 
All fine-tuning methods improve the generalization ability, but only fine-tuning on the non-robust-critical module ({\color{mypink}{\ttfamily layer2.1.conv2}}) can preserve robustness. Additionally, fine-tuning on robust-critical module ({\color{mypurple}{\ttfamily layer4.1.conv1}}) causes the worst trade-off between generalization and robustness. In the initial interpolation stage, fine-tuning on non-robust-critical modules enhances adversarial robustness by around $0.3\%$.}
\label{fig: interpolation}
\end{figure}

The pursuit of accurate and trustworthy artificial intelligence systems is a fundamental objective in the deep learning community.
Adversarial examples \cite{szegedy2013intriguing, goodfellow15}, which perturbs input by a small, human imperceptible noise that can cause deep neural networks to make incorrect predictions, pose a significant threat to the security of AI systems.
Notable experimental and theoretical progress has been made in defending against such adversarial examples \cite{carlini17, athalye18, cohen19c, hendrycks2019using, aa2020, gowal21generated, rebuffi2021fixing}.
Among various defense methods \cite{xie2018mitigating, Pang2020Rethinking, zhang20featurescattering, mustafa2019hidden, Chan2020Jacobian}, adversarial training (AT) \cite{madry2018towards} has been shown to be one of the most promising approaches~\cite{athalye18, aa2020} to enhance the adversarial robustness.
However, compared to standard training, AT severely sacrifices generalization on in-distribution data \cite{schmidt18moredata, tsipras2018robustness, trades, raghunathan20a, score} and is exceptionally vulnerable to certain out-of-distribution (OOD) examples \cite{gilmer2019adversarial, yi21improved, kireev2022effectiveness} such as Contrast, Bright and Fog, resulting in unsatisfactory performance.

Prior arts tend to mitigate the trade-off between generalization and adversarial robustness within the adversarial training procedure.
For example, some approaches have explored reweighting instances \cite{zhang2021geometryaware}, using unlabeled data \cite{raghunathan20a}, or redefining the robust loss function \cite{trades, mart, awp, score}.
In this paper, we take a different perspective to address such a trade-off by leveraging the \emph{redundant capacity for robustness} of neural networks after adversarial training.
Recent research has demonstrated that deep neural networks can exhibit \emph{redundant capacity for generalization} due to their complex and opaque nature, where specific network modules can be deleted, permuted \cite{veit2016residual}, or reset to their initial values \cite{zhang2017understanding, chatterji2020The} with only minor degradation in generalization performance.
Hence, it is intuitive to ask: \textit{Do adversarially trained models have such redundant capacity? If so, how to leverage it to improve the generalization and OOD robustness \footnote{Here, generalization refers to generalization to in-distribution (ID) samples, and OOD robustness refers to generalization to OOD samples.} while maintaining adversarial robustness?}

Based on such motivation, we introduce a new concept called \emph{Module Robust Criticality (MRC)} \footnote{In our paper, a module refers to a layer of the neural network.} to investigate the redundant capacity of adversarially trained models for robustness.
MRC aims to quantify the \emph{maximum increment of robustness loss} of a module's parameters under the \emph{constrained weight perturbation}.
As illustrated in \figurename~\ref{fig: module robust criticality}, we empirically find that certain modules do exhibit redundant characteristics under such perturbations, resulting in negligible drops in adversarial robustness.
We refer to the modules with the lowest MRC value as the \emph{non-robust-critical modules}.
These findings further inspire us to propose a novel fine-tuning technique called \textbf{R}obust Cr\textbf{i}tical \textbf{F}ine-\textbf{T}uning (\method), which aims to leverage the redundant capacity of the non-robust-critical module to improve generalization while maintaining adversarial robustness.
\method consists of three steps: (1) Module robust criticality characterization, which calculates the MRC value for each module and identifies the non-robust-critical module.
(2) Non-robust-critical module fine-tuning, which exploits the redundant capacity of the non-robust-critical module via fine-tuning its weights with standard examples.
(3) Mitigating robustness-generalization trade-off via interpolation, which interpolates between adversarially trained parameters and fine-tuned parameters to find the best weights that maximize the improvement in generalization while preserving adversarial robustness.

Experimental results demonstrate that \method significantly improves both the generalization performance and OOD robustness by around $2$\% while maintaining or even improving the adversarial robustness of the original models.
Furthermore, we also incorporate \method to other adversarial training regimes such as TRADES \cite{trades}, MART \cite{mart}, AT-AWP \cite{awp}, and SCORE \cite{score}, and show that such incorporation leads to further enhancements. More importantly, our experiments reveal several noteworthy insights.
\emph{First}, we found that fine-tuning on non-robust-critical modules can effectively mitigate the trade-off between adversarial robustness and generalization, showing that these two can both be improved (Section \ref{sec: 5.3 main results}).
As illustrated in \figurename~\ref{fig: interpolation}, adversarial robustness increases alongside the generalization in the initial interpolation procedure, indicating that the features learned by fine-tuning can benefit both generalization and adversarial robustness. 
This contradicts the previous claim~\cite{tsipras2018robustness} that the features learned by optimal standard and robust classifiers are fundamentally different.
\emph{Second}, the existence of non-robust-critical modules suggests that current adversarial training regimes do not fully utilize the capacity of DNNs (Section~\ref{sec 5.2: mrc analysis}).
This motivates future work to design more efficient adversarial training approaches using such capacity.
\emph{Third}, while previous study \cite{kumar2022finetuning} reported that fine-tuning on \emph{pre-train models} could distort the learned robust features and result in poor performance on OOD samples, we find that fine-tuning \emph{adversarially trained models} do \emph{NOT} lead to worse OOD performance (Section \ref{sec: 5.3 main results}).

The contribution of this work is summarized as follows:
\begin{enumerate}
    \item \textbf{Novel approach.} We propose the concept of module robust criticality and verify the existence of redundant capacity for robustness in adversarially trained models. We then propose \method to exploit such redundancy to improve the generalization of AT models.
    \item \textbf{Superior experimental results.} Our approach improves both generalization and OOD robustness of AT models by around $2$\%. It can also be incorporated with previous AT methods to mitigate the trade-off between generalization and adversarial robustness.
    \item \textbf{Interesting insights.} The findings of our experiments shed light on the intricate interplay between generalization, adversarial robustness, and OOD robustness. Our work emphasizes the potential of leveraging the redundant capacity in adversarially trained models to improve generalization and robustness further, which may inspire more efficient and effective training methods to fully utilize this redundancy.
\end{enumerate}


\section{Related Work}

\paragraph{Trade-off between adversarial robustness and generalization}

The existence of such trade-off has been extensively debated in the adversarial learning community \cite{schmidt18moredata, tsipras2018robustness, trades, javanmard20precise, raghunathan20a, score}.
Despite lingering controversies, the prevalent viewpoint is that this trade-off is inherent. 
Theoretical analyses \cite{tsipras2018robustness, raghunathan20a, javanmard20precise} demonstrated that the trade-off provably exists even in simple cases, \eg, binary classification and linear regression. 
To address this trade-off, various methods have been proposed during adversarial training, such as instance reweighting \cite{zhang2021geometryaware}, robust self-training \cite{raghunathan20a}, incorporating unlabeled data \cite{carmon19unlabel, hendrycks2019using}, and redefining the robust loss function \cite{trades, mart, awp, score}. 
This paper presents a novel post-processing approach that exploits the excess capacity of the model after adversarial training to address such trade-off.
Our \method can be used in conjunction with existing adversarial training techniques, providing a practical and effective way to mitigate the trade-off further.

\paragraph{Redundant Fitting Capacity}
The over-parameterized deep neural networks (DNNs) exhibit striking fitting power even for random labels \cite{zhang2017understanding, arpit2017closer}. 
Recent studies have shown that not all modules contribute equally to the generalization ability of DNNs \cite{veit2016residual, rosenfeld2019intriguing, zhang2019all, chatterji2020The}, indicating the redundant fitting capacity for generalization. Veit \etal \cite{veit2016residual} found that some blocks can be deleted or permuted without degrading the test performance too much. 
Rosenfeld and Tsotsos \cite{rosenfeld2019intriguing} demonstrated that one could achieve comparable performance by training only a small fraction of network parameters.
Further, recent studies have identified certain neural network modules, referred to as \textit{robust modules} \cite{zhang2019all, chatterji2020The}, rewinding their parameters to initial values results in a negligible decline in generalization. Previous studies have proposed methods to reduce the computational and storage costs of deep neural networks by \emph{removing} the redundant capacity for generalization while preserving comparable performance, such as compression \cite{han2015deep} and distillation \cite{hinton2015distilling}. In contrast, our work focuses on the \emph{redundant capacity for robustness} of adversarially trained models and tries to \emph{exlpoit} such redundancy.

\paragraph{Fine-tuning Methods}
Pre-training on large scale datasets has been shown to be a powerful approach for developing high-performing deep learning models \cite{brown2020language, dosovitskiy2021an, radford2019language, kolesnikov2020big}. Fine-tuning is a widely adopted approach to enhance the transferability of pre-trained models to downstream tasks and domain shifts. Typically, fine-tuning methods involve fine-tuning the last layer (linear probing) \cite{aghajanyan2021better, kumar2022finetuning} or all layers (fully fine-tuning) \cite{aghajanyan2021better, hendrycks2019using, miller21bline, kumar2022finetuning}. Salman \etal \cite{salman2020adversarially} demonstrated that both fully fine-tuning and linear probing of adversarially trained models can improve the transfer performance on downstream tasks.
Nevertheless, recent studies \cite{andreassen2021evolution, wiseft, kumar2022finetuning} have suggested that fine-tuning can degrade pre-trained features and underperformance on out-of-distribution (OOD) samples.
To address this issue, different fine-tuning techniques are proposed such as WiSE-FT~\cite{wiseft} and surgical fine-tuning~\cite{lee2023surgical} that either leveraged ensemble learning or selective fine-tuning for better OOD performance.
Kumar \etal \cite{kumar2022finetuning} suggested the two-step strategy of linear probing then full fine-tuning (LP-FT) combines the benefits of both fully fine-tuning and linear probing.

\section{Module Robust Criticality}


Improving the generalization of adversarially trained models requires a thorough understanding of DNNs, which, however, proves to be difficult due to the lack of explainability.
Luckily, recent studies show that specific modules in neural networks, referred to as \textit{critical modules} \cite{zhang2019all, chatterji2020The}, significantly impact model generalization if their parameters are rewound to initial values.
In this work, we propose a metric called \textbf{Module Robust Criticality (MRC)} to evaluate the robustness contribution of each module explicitly.

\subsection{Preliminaries}
We denote a $l$-layered DNN as $f(\vtheta) = \phi(\vx^{(l)}; \vtheta^{(l)}) \circ \hdots \circ \phi(\vx^{(1)}; \vtheta^{(1)})$, where $\vtheta^{(i)}$ is the parameter of $i$-th layer and $\phi(\cdot)$ denotes the activation function.
We use $\vtheta_{AT}$ and $\vtheta_{FT}$ to denote the weights of the adversarially trained and fine-tuned model, respectively.
We use $\mathcal{D} = \{(\vx_1, y_1), ..., (\vx_n, y_n)\}$ to denote a dataset and $\mathcal{D}_{\mathit{std}}$ means a standard dataset such as CIFAR10. The cross-entropy loss is denoted by $\mathcal{L}$ and $\lVert \cdot \rVert_p$ is denoted as the $\ell_p$ norm. 

Let $\Delta \vx \in \mathcal{S}$ denote the adversarial perturbation applied to a clean input $\vx$, where $\mathcal{S}$ represents the allowed range of input perturbations. Given a neural network $f(\vtheta)$ and a dataset $\mathcal{D}$, adversarial training aims to minimize the robust loss \cite{madry2018towards} as: 
\begin{equation}
\label{eq-at}
\begin{split}
& \mathop{\arg \min} \limits_{\vtheta} \mathcal{R}(f(\vtheta), \mathcal{D}), \text{ where }\\
\mathcal{R}(f(\vtheta), \mathcal{D}) = &\sum \limits_{(\vx, y) \in \mathcal{D} } \max_{\Delta \vx \in \mathcal{S}} \mathcal{L} (f ( \vtheta, \vx + \Delta \vx), y).
\end{split}
\end{equation}

Here, $\mathcal{R}(f(\vtheta), \mathcal{D})$ is the robust loss to find the \emph{worst-case input perturbation} that maximizes the cross-entropy classification error.

\subsection{Module Robust Criticality}
\label{section 3.1}



\begin{definition}[Module Robust Criticality]
\label{def: mrc}
Given a weight perturbation scaling factor $\epsilon > 0$ and a neural network $f(\vtheta)$, the robust criticality of a module $i$ is defined as
\begin{align}
\label{eq: mrc}
    \mathit{MRC} (f, \vtheta^{(i)}, \mathcal{D}, \epsilon) = 
    &\max \limits_{\Delta \vtheta \in \mathcal{C}_{\vtheta}} \mathcal{R} (f (\vtheta+\Delta \vtheta), \mathcal{D}) \nonumber \\
    & - \mathcal{R} (f (\vtheta), \mathcal{D}),
\end{align} 
where $\Delta \vtheta = \{\mathbf{0}, \hdots, \mathbf{0}, \Delta \vtheta^{(i)}, \mathbf{0}, \hdots, \mathbf{0}\}$ denotes the weight perturbation with respect to the module weights $\vtheta^{(i)}$, $\mathcal{C}_{\vtheta} = \{ \Delta \vtheta \, \big|  \, \lVert \Delta \vtheta \rVert_p \leq \epsilon \lVert \vtheta^{(i)} \rVert_p \}$, $\mathcal{R}(\cdot)$ is the robust loss defined in \eq{eq-at}.
\end{definition}

The MRC value for each module represents how they are critically contributing to model adversarial robustness. The module with the lowest MRC value is considered redundant, as changing its weights has a negligible effect on robustness degradation. We refer to this module as the non-robust-critical module. Intuitively, MRC serves as an upper bound for weight changing of a particular module, as demonstrated in Theorem \ref{theorem: mrc}. Since we do not know the optimization directions and how they might affect the model robustness to adversarial examples, we measure the extent to which \emph{worst-case} weight perturbations affect the robustness, providing an upper bound loss for optimizing the weight.
Further, the MRC for a module depicts the sharpness of robust loss landscape \cite{awp, stutz2021relating} around the minima $\vtheta^{(i)}$. If the MRC score is high, it means that the robust loss landscape with respect to $\vtheta^{(i)}$ is sharp, and fine-tuning this module is likely to hurt the adversarial robustness. 

\begin{theorem}
\label{theorem: mrc}
The MRC value for a module $i$ serves as an upper bound for the robust loss increase when we optimize the module under constraint $\mathcal{C}_{\vtheta}$:
\begin{gather}
    \mathcal{R}( f(\vtheta^*), \mathcal{D}) - \mathcal{R}( f(\vtheta), \mathcal{D}) \leq \mathit{MRC} (f, \vtheta^{(i)}, \mathcal{D}, \epsilon), \nonumber  \\ 
    \text{where } \, \vtheta^* = \mathop{\arg \min} \limits_{\vtheta', (\vtheta' - \vtheta) \in \mathcal{C}_{\vtheta}} \sum \limits_{(\vx, y) \in \mathcal{D}} \mathcal{L}(f(\vtheta', x), y).
\end{gather}
\end{theorem}

\begin{proof} By the definition of MRC, for any weights $(\vtheta'-\vtheta) \in \mathcal{C}_{\vtheta}$, we have:
\begin{equation}
    \mathcal{R}( f(\vtheta'), \mathcal{D}) - \mathcal{R}( f(\vtheta), \mathcal{D}) \leq \mathit{MRC} (f, \vtheta^{(i)}, \mathcal{D}, \epsilon).
\end{equation}
Thus, for the optimized weights:
\begin{gather}
    \vtheta^* = \mathop{\arg \min} \limits_{\vtheta', (\vtheta' - \vtheta) \in \mathcal{C}_{\vtheta}} \sum \limits_{(\vx, y) \in \mathcal{D}} \mathcal{L}(f(\vtheta', x), y),
\end{gather}
it satisfies 
\begin{gather}
    \mathcal{R}( f(\vtheta^*), \mathcal{D}) - \mathcal{R}( f(\vtheta), \mathcal{D}) \leq \mathit{MRC} (f, \vtheta^{(i)}, \mathcal{D}, \epsilon).
\end{gather}
Such that the proof ends.
\end{proof}


\paragraph{Remark:}
The definition of MRC is similar in spirit to the work of Zhang \etal \cite{zhang2019all} and Chatterji \etal \cite{chatterji2020The}. However, 
MRC differs fundamentally from them in two aspects.
First, MRC aims to capture the influence of a module on \textit{adversarial robustness}, while Zhang \etal \cite{zhang2019all} and Chatterji \etal \cite{chatterji2020The} focus on studying the impact of a module on \textit{generalization}.
Second, MRC investigates the robustness characteristics of module weights under \textit{worst-case weight perturbations}, whereas Zhang \etal \cite{zhang2019all} and Chatterji \etal \cite{chatterji2020The} analyzed the properties of a module by \textit{rewinding its weights to their initial values}.
Similar to \cite{asam, stutz2021relating}, we define the weight perturbation constraint $C_{\vtheta}$ as a multiple of the $\ell_p$ norm of original parameters, which ensures the scale-invariant property and allows us to compare the robust criticality of modules across different layers, see Appendix~\ref{append-proof} for a detailed proof.

Theorem \ref{theorem: mrc} establishes a clear upper bound for fine-tuning particular modules. This theorem assures us that fine-tuning on non-robust-critical modules shouldn't harm the model robustness. However, it does not ascertain if fine-tuning the robust-critical module will lead to a significant decline in robust accuracy.

\subsection{Relaxation of MRC}

Optimizing in \eq{eq: mrc} requires simultaneously finding worst-case weight perturbation $\Delta \vtheta$ and worst-case input perturbation $\Delta \vx$, which is time-consuming.
Thus, we propose a relaxation version by fixing $\Delta \vx$ at the initial optimizing phase.
Concretely, we first calculate the adversarial examples $\Delta \vx$ with respect to $\vtheta_{AT}$. By fixing the adversarial examples unchanged during the optimization, we iteratively optimize the $\Delta \vtheta$ by gradient ascent method to maximize the robust loss to find the optimal $\Delta \vtheta$.
We set a weight perturbation constraint and check it after each optimization step. If the constraint is violated, we project the perturbation onto the constraint set. The pseudo-code is described in Algorithm \ref{alg: mrc calculation}.
In our experiments, if not specified, we set $\lVert \cdot \rVert_p = \lVert \cdot \rVert_2$ and $\epsilon = 0.1$ for $C_{\vtheta}$, the iterative step for optimizing $\Delta \vtheta$ is $10$.


\begin{algorithm}[t!]
\caption{Module Robust Criticality Characterization}
\label{alg: mrc calculation}
\begin{algorithmic}[1]
\Require neural network $f$, adversarially trained model weights $\vtheta_{AT}$, desired module $i$'s weights $\vtheta^{(i)}$, standard dataset $\mathcal{D}_{std}$, weight perturbation scaling factor $\epsilon$, optimization iteration steps $T$, learning rate $\gamma$.
\Ensure The module robust criticality of module $i$.

\State{Initialize adversarial dataset: $\mathcal{D}_{adv} = \{\}$} 
\For{Batch $\mathcal{B}_k \in \mathcal{D}_{std}$}  \Comment{Generate adversarial dataset}
\State{$\mathcal{B}_k^{adv}$ = PGD-10($\vtheta_{AT}, \mathcal{B}_k$)} 
\State{$\mathcal{D}_{adv} = \mathcal{D}_{adv} \bigcup \mathcal{B}_k^{adv} $}
\EndFor

\State{Freeze all parameters of $\vtheta_{AT}$ except for $\vtheta^{(i)}$}
\State{$\vtheta_{1} = \vtheta_{AT}$}
\For{$t = 1, \hdots, T$} \Comment{Iterate $T$ epochs}
\State{$\vtheta_{t+1} = \vtheta_{t}$}
\For{Batch $\mathcal{B}_k^{adv} \in \mathcal{D}_{adv}$}
\State{Calculate Loss: $\mathcal{L}(f, \vtheta_t, \mathcal{B}_k^{adv}))$}
\State{$\vtheta_{t+1} = \vtheta_{t+1} + \gamma \nabla_{\vtheta_t}(\mathcal{L})$} \Comment{Gradient Ascent}
\EndFor
\State{$\Delta \vtheta^{(i)} = \vtheta_{t+1}^{(i)} - \vtheta_{AT}^{(i)}$} \Comment{Check perturb constraint}
\If{$\lVert \Delta \vtheta^{(i)} \rVert_2 \geq \epsilon \lVert \vtheta_{AT}^{(i)} \rVert_2$}

\State{$\Delta \vtheta^{(i)} = \epsilon \frac{\lVert \vtheta_{AT}^{(i)} \rVert_2}{ \lVert \Delta \vtheta^{(i)} \rVert_2}  \Delta \vtheta^{(i)} $}   
\State{$\vtheta_{t+1} = \vtheta_{t} + \Delta \vtheta^{(i)}$}
\Break
\EndIf
\EndFor
\State{$\mathit{MRC}(\vtheta^{(i)}) = \mathcal{L}(f, \vtheta_{T}, \mathcal{D}_{\mathit{adv}}) - \mathcal{L}(f, \vtheta_{AT}, \mathcal{D}_{\mathit{adv}})$}

\State{\textbf{Return $\mathit{MRC}(\vtheta^{(i)})$}}
\end{algorithmic}
\end{algorithm}

\begin{figure*}[t!]
\centering
\includegraphics[width=\textwidth]{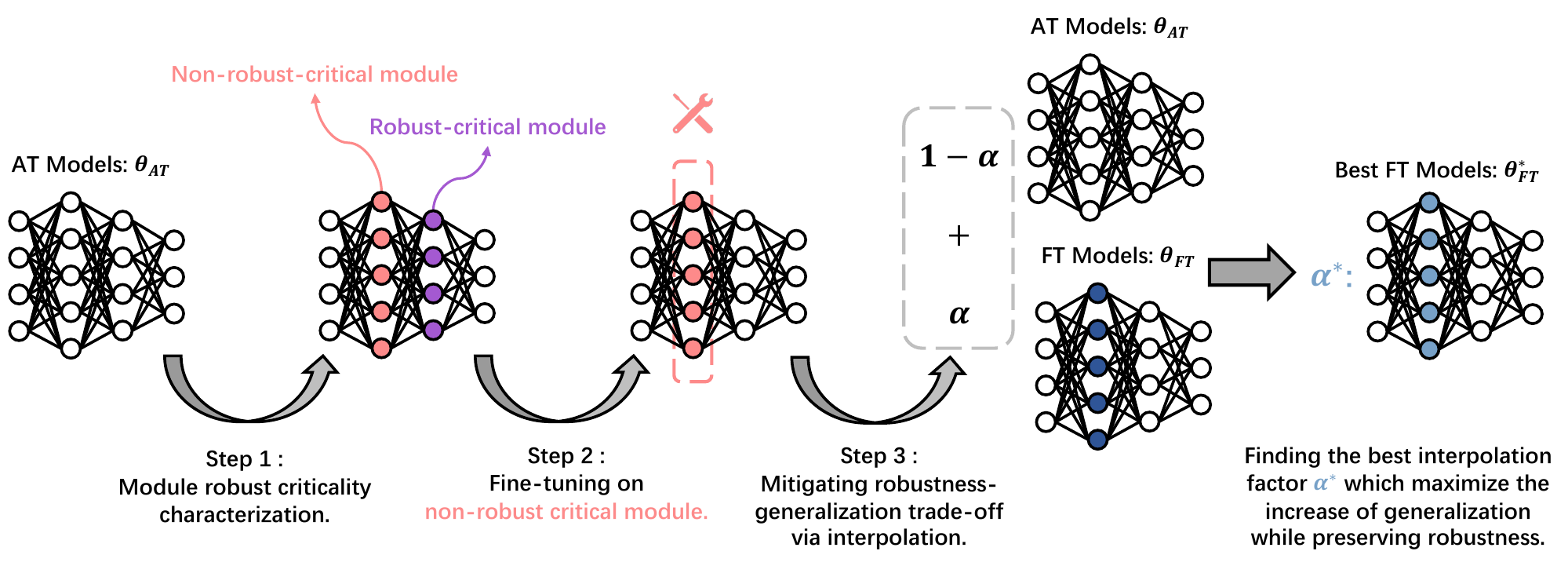}
\caption{The pipeline of our proposed Robust Critical Fine-Tuning (\method).}
\label{fig: rift}
\end{figure*}

\section{\method: Robust Critical Fine-tuning}
\label{sec: rift}
In this paper, we propose \textbf{\method}, a robust critical fine-tuning approach that leverages MRC to guide the fine-tuning of a deep neural network to improve both generalization and robustness.
Let $\mathcal{P}_{\mathit{adv}}(x, y)$ and $\mathcal{P}_{\mathit{std}}(x, y)$ denote the distributions of adversarial and standard inputs, respectively.
Then, applying an adversarially trained model on $\mathcal{P}_{\mathit{adv}}(x, y)$ to $\mathcal{P}_{\mathit{std}}(x, y)$ can be viewed as a \emph{distributional shift} problem.
Thus, it is natural for \method to exploit the redundant capacity to fine-tune adversarially trained models on the standard dataset.

Specifically, \method consists of three steps as shown in \figurename~\ref{fig: rift}.
First, we calculate the MRC of each module and choose the module with the lowest MRC score as our non-robust-critical module.
Second, we freeze the parameters of the adversarially trained model except for our chosen non-robust-critical module. Then we fine-tune the adversarially trained models on corresponding standard dataset $\mathcal{D}_{\mathit{std}}$.
Third, we linearly interpolate the weights of the original adversarially trained model and fine-tuned model to identify the optimal interpolation point that maximizes generalization improvement while maintaining robustness. 


\paragraph{Step 1: Module robust criticality characterization}

According to the Algorithm~\ref{alg: mrc calculation}, we iteratively calculate the MRC value for each module $\vtheta^{(i)} \in \vtheta_{AT}$, then we choose the module with the lowest MRC value, denoted as $\tilde{\vtheta}$:
\begin{equation}
    \tilde{\vtheta} = \vtheta^{(i)} \, \text{ where } \, i = \mathop{\arg \min} \limits_i {\mathit{MRC}(f, \vtheta^{(i)}, \mathcal{D}, \epsilon)}.
\end{equation}

\paragraph{Step 2: Fine-tuning on non-robust-critical modules}
Next, we freeze the rest of the parameters and fine-tune on desired parameters $\tilde{\vtheta}$.
We solve the following optimization problem by SGD with momentum \cite{sutskever2013importance}
\begin{align}
\label{eq: fine-tuning}
    \mathop{\arg \min} \limits_{\tilde{\vtheta}} \sum \limits_{(\vx, y) \in \mathcal{D}} \mathcal{L} ( f (x, (\tilde{\vtheta}; \vtheta \setminus \tilde{\vtheta})), y)  + \lambda  \lVert \tilde{\vtheta} \rVert_2,
\end{align}
where $\lambda$ is the $\ell_2$ weight decay factor.

\paragraph{Step 3: Mitigating robustness-generalization trade-off via interpolation}
 
For a interpolation coefficient $\alpha$, the interpolated weights is calculated as:
\begin{align}
\label{eq: interpolation}
    \vtheta_\alpha = (1-\alpha) \vtheta_{AT} + \alpha \vtheta_{FT},
\end{align}
where $\vtheta_{AT}$ is the initial adversarially trained weights and $\vtheta_{FT}$ is the fine-tuned weights obtained by \eq{eq: fine-tuning}. Since our goal is to improve the generalization while preserving adversarial robustness, thus the best interpolation point is chosen to be the point that most significantly improves the generalization while the corresponding adversarial robustness is no less than the original robustness by $0.1$.

\paragraph{Remark:}
Theorem \ref{theorem: mrc} establishes an upper bound on the possible drop in robustness loss that can be achieved through fine-tuning. It is expected that the second step of optimization would enforce the parameters to lie within the boundary $\mathcal{C}_{\vtheta}$ in order to satisfy the theorem. However, here we do not employ constrained optimization but find the optimal point by first optimizing without constraints and then interpolating. This is because (1) the constraints are empirically given and may not always provide the optimal range for preserving robustness, and it is possible to fine-tune outside the constraint range and still ensure that there is not much loss of robustness. (2) the interpolation procedure serves as a weight-ensemble, which may benefit both robustness and generalization, as noted in WiSE-FT \cite{wiseft}.
The complete algorithm of \method is shown in Appendix~\ref{append-algo}.



\section{Experiments}


\subsection{Experimental Setup}
\label{section: 4.1}

\begin{figure*}[ht]
\centering
\includegraphics[width=\textwidth]{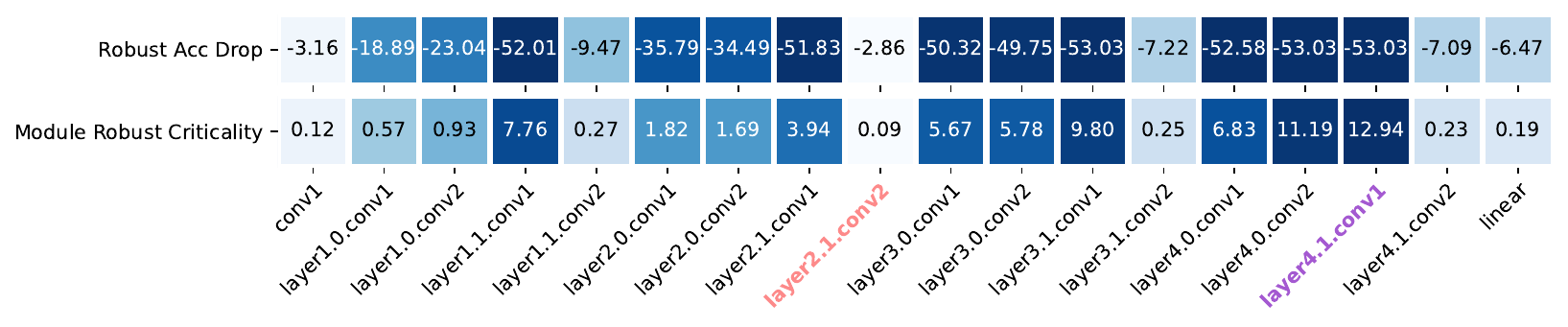}
\caption{Example of module robust criticality (MRC) and its corresponding robust accuracy drop of ResNet18 trained on CIFAR10. 
Each column represents an individual module. 
The first row represents the corresponding robust accuracy drop and the second row represents the MRC value of each module.
The higher the MRC value is, the more robust-critical the module is. 
Some modules are not critical to robustness, exhibiting redundant characteristics for contributing to robustness. 
However, some modules are critical to robustness. 
For example, the robust acc drop is only $2.86\%$ for {\color{mypink}{\ttfamily layer2.1.conv2}} while for {\color{mypurple}{\ttfamily layer4.1.conv1}} the robust acc drop is up to $53.03\%$. }
\label{fig: module robust criticality}
\end{figure*}

\paragraph{Datasets}
We adopt three popular image classification datasets: CIFAR10 \cite{cifar}, CIFAR100 \cite{cifar}, and Tiny-ImageNet \cite{tiny}.
CIFAR10 and CIFAR100 comprise 60,000 $32\times32$ color images in 10 and 100 classes, respectively. 
Tiny-ImageNet is a subset of ImageNet and contains $200$ classes, where each class contains 500 colorful images with size $64 \times 64$.
We use three OOD datasets accordingly to evaluate the OOD robustness: CIFAR10-C, CIFAR100-C, and Tiny-ImageNet-C \cite{hendrycks2019using}.
These datasets simulate $15$ types of common visual corruptions and are grouped into four classes: Noise, Blur, Weather, and Digital.

\paragraph{Evaluation metrics}
We use the test set accuracy of each standard dataset to represent the generalization ability. For evaluating adversarial robustness, we adopt a common setting of PGD-10 \cite{madry2018towards} with constraint $\ell_\infty = 8/255$. We run PGD-10 with three times and select the \emph{worst} robust accuracy as the final metric. The OOD robustness is evaluated by the accuracy of the test set of the corrupted dataset corresponding to the standard dataset.

\paragraph{Training details}
We use ResNet18 \cite{he2016deep}, ResNet34 \cite{he2016deep}, WideResNet34-10 (WRN34-10) \cite{wrn} as backbones.
ResNet18 and ResNet34 are 18-layer and 34-layer ResNet models, respectively.
WideResNet34-10 is a 34-layer WideResNet model with a widening factor of 10.
Similarly, we adopt PGD-10 \cite{madry2018towards} with constraint $\ell_\infty = 8/255$ for adversarial training. 
Following standard settings \cite{rice20aoverfitting, pang2021bag}, we train models with adversarial examples for $110$ epochs.
The learning rate starts from $0.1$ and decays by a factor of $0.1$ at epochs $100$ and $105$.
We select the weights with the highest test robust accuracy as our adversarially trained models.

We fine-tune the adversarially trained models $\vtheta_{AT}$ using SGD with momentum \cite{sutskever2013importance} for $10$ epochs.
The initial learning rate is set to $0.001$.\footnote{The best learning rate for fine-tuning vary across architectures and datasets and is required to be carefully modified.}
We decay the learning rate by $1/10$ after fine-tuning for $5$ epochs
We choose the weights with the highest test accuracy as fine-tuned model weights, denoted as $\vtheta_{FT}$.
We then interpolate between initial adversarially trained model weights $\vtheta_{AT}$ and $\vtheta_{FT}$, the best interpolation point selected by Step 3 in Section \ref{sec: rift} is denoted as $\vtheta_{FT}^*$. We then compare the generalization, adversarial robustness, and OOD robustness of $\vtheta_{FT}^*$ and $\vtheta_{AT}$.

We report the average of three different seeds and omit the standard deviations of 3 runs as they are tiny ($< 0.20$\%), which hardly effect the results. 
Refer to Appendix \ref{append-training details} for more training details.

\begin{table*}[t!]
\caption{Results of \method on different datasets and backbones. \textit{Std} means the standard test accuracy for in distribution generalization, \textit{OOD} denotes the OOD robust accuracy of corresponding corruption dataset (\eg, CIFAR10-C). \textit{Adv} denotes the adversarial robust accuracy. In each column, we bold the entry with the higher accuracy. \method improves both generalization and OOD robustness across architectures and datasets while maintaining adversarial robustness.}
\label{table: fine-tuning results}
\centering  
\resizebox{.8\textwidth}{!}{
\begin{tabular}{ c  c  c c c  c c c  c c c}
\toprule

\multirow{2}{*}{Architecture} & \multirow{2}{*}{Method} & \multicolumn{3}{c}{CIFAR10} & \multicolumn{3}{c}{CIFAR100} & \multicolumn{3}{c}{Tiny-ImageNet} \\
\cmidrule(lr){3-5} \cmidrule(lr){6-8} \cmidrule(lr){9-11}
 & & \textit{Std} & \textit{OOD} & \textit{Adv} & \textit{Std} & \textit{OOD} & \textit{Adv} & \textit{Std} & \textit{OOD} & \textit{Adv}\\
\midrule
\multirow{3}{*}{ResNet18}  & AT & 81.46 & 73.56 & 53.63 & 57.10 & 46.43 & 30.15 & 49.10 & 27.68 & 23.28\\
          & AT+RiFT & \textbf{83.44} & \textbf{75.69} & \textbf{53.65} & \textbf{58.74} & \textbf{48.06} & \textbf{30.17} & \textbf{50.61} & \textbf{28.73} & \textbf{23.34} \\
    & $\Delta$ & \color{Red} +1.98 & \color{Red} +2.13 & \color{Red} +0.02 & \color{Red} +1.64 & \color{Red} +1.63 & \color{Red} +0.02 & \color{Red} +1.51 & \color{Red} +1.05 & \color{Red} +0.06\\
\midrule
\multirow{3}{*}{ResNet34}  & AT & 84.23 & 75.37 & 55.31 & 58.67 & 48.24 & 30.50 & 50.96 & 27.91 & 24.27\\
          & AT+RiFT  & \textbf{85.41} & \textbf{77.15} & \textbf{55.34} & \textbf{60.88} & \textbf{49.97} & \textbf{30.58} & \textbf{52.54} & \textbf{30.07} & \textbf{24.37}\\
& $\Delta$ & \color{Red} +1.18 & \color{Red} +1.78 & \color{Red} +0.03  & \color{Red} +2.21 & \color{Red} +1.73  & \color{Red} +0.08 & \color{Red} +1.58 & \color{Red} +2.16 & \color{Red} +0.10  \\
\midrule
\multirow{3}{*}{WRN34-10}    & AT & 87.41 & 78.75 & 55.40 & 62.35 & 50.61 & \textbf{31.66} & 52.78 & 31.81 & 26.07  \\
                & AT+RiFT & \textbf{87.89} & \textbf{79.31} & \textbf{55.41} & \textbf{64.56} & \textbf{52.69} & 31.64 & \textbf{55.31} & \textbf{33.86} & \textbf{26.17}\\
& $\Delta$ & \color{Red} +0.48 & \color{Red} +0.56 & \color{Red} +0.01 & \color{Red} +2.21 & \color{Red} +2.08 & \color{Green} -0.02 & \color{Red} +2.53 & \color{Red} +2.05 & \color{Red} +0.10\\
\midrule
Avg & $\Delta$ & \color{Red} +1.21 & \color{Red} +1.49 & \color{Red} +0.02 & \color{Red} +2.02 & \color{Red} +1.81 & \color{Red} +0.02 & \color{Red} +1.87 & \color{Red} +1.75 & \color{Red} +0.08\\ 
\bottomrule
\end{tabular}
}
\end{table*}

\subsection{Empirical Analysis of MRC}
\label{sec 5.2: mrc analysis}
Before delving into the main results of \method, we first empirically analyze our proposed MRC metric in Definition~\ref{def: mrc}, which serves as the foundation of our \method approach.
We present the MRC analysis on ResNet18 \cite{he2016deep} on CIFAR-10 in \figurename~\ref{fig: module robust criticality}, where each column corresponds to the MRC value and its corresponding robust accuracy drop of a specific module.

Our analysis shows that the impact of worst-case weight perturbations on model robustness varies across different modules. 
Some modules exhibit minimal impact on robustness under perturbation, indicating the presence of \emph{redundant capacity for robustness}. Conversely, for other modules, the worst-case weight perturbation shows a significant impact, resulting in a substantial decline in robustness.
For example, in module {\color{mypink}{\ttfamily layer2.1.conv2}}, worst-case weight perturbations only result in a meager addition of $0.09$ robust loss. However, for {\color{mypurple}{\ttfamily layer4.1.conv1}}, the worst-case weight perturbations affect the model's robust loss by an additional $12.94$, resulting in a substantial decline ($53.03\%$) in robustness accuracy.
Such robust-critical and non-robust-critical modules are verified to exist in various network architectures and datasets, as detailed in Appendix~\ref{append-mrc value plots}.
We also observe that as the network capacity decreases (\eg, from WRN34-10 to ResNet18) and the task becomes more challenging (\eg, from CIFAR10 to Tiny-ImageNet), the proportion of non-robust-critical modules increases, as less complex tasks require less capacity, leading to more non-robust-critical modules.

It is worthy noting that the decrease in robust accuracy doesn't directly correlate with MRC. For instance, both {\ttfamily layer4.0.conv2} and {\ttfamily layer4.1.conv1} have a robust accuracy drop of $53.05\%$, yet their MRC values differ. This discrepancy can be attributed to the different probability distributions of misclassified samples across modules, resulting in same accuracy declines but different losses.

\subsection{Main Results}
\label{sec: 5.3 main results}

\tablename~\ref{table: fine-tuning results} summarizes the main results of our study, from which we have the following findings.

\paragraph{RiFT improves generalization}
First, \method effectively mitigates the trade-off between generalization and robustness raised by adversarial training. Across different datasets and network architectures, \method improves the generalization of adversarially trained models by approximately $2$\%. This result prompts us to rethink the trade-off, as it may be caused by inefficient adversarial training algorithm rather than the inherent limitation of DNNs.
Furthermore, as demonstrated in \figurename~\ref{fig: interpolation}, both adversarial robustness and generalization increase simultaneously in the initial interpolation process, indicating that these two characteristics can be improved together.
This trend is observed across different datasets and network architectures; see Appendix~\ref{append-interpolation} for more illustrations.
This finding challenges the notion that the features of optimal standard and optimal robust classifiers are fundamentally different, as previously claimed by Tsipras \etal \cite{tsipras2018robustness}, as fine-tuning procedures can increase both robustness and generalization.

\paragraph{Fine-tuning improves OOD robustness}
Second, our study also investigated the out-of-distribution (OOD) robustness of the fine-tuned models and observed an improvement of approximately $2$\%.
This observation is noteworthy because recent work \cite{andreassen2021evolution, kumar2022finetuning, wiseft} showed that fine-tuning pre-trained models can distort learned features and result in underperformance in OOD samples.
Furthermore, Yi \etal \cite{yi21improved} demonstrated that adversarial training enhances OOD robustness, but it is unclear whether fine-tuning on adversarially trained models distorts robust features.
Our results indicate that fine-tuning adversarially trained models does not distort the robust features learned by adversarial training and instead helps improve OOD robustness. We suggest fine-tuning adversarially trained models may be a promising avenue for further improving OOD robustness.

\subsection{Incorporate \method to Other AT Methods}

\begin{table*}[t!]
\parbox{.57\linewidth}
{
\caption{Results of \method + other AT methods.}
\label{table: combine other tricks}
\centering
\resizebox{.57\textwidth}{!}{
\begin{tabular}{ c c c c c c c c c c}

\toprule

\multirow{2}{*}{Method}  & \multicolumn{3}{c}{CIFAR10} & \multicolumn{3}{c}{CIFAR100}  \\
\cmidrule(lr){2-4}  \cmidrule(lr){5-7} 
& \textit{Std} & \textit{OOD} & \textit{Adv} & \textit{Std} & \textit{OOD} & \textit{Adv} \\

\midrule
TRADES & 81.54 & 73.42 & \textbf{53.31} & 57.44 & 47.23 & 30.20 \\
TRADES+RiFT & \textbf{81.87} & \textbf{74.09} & 53.30 & \textbf{57.78} & \textbf{47.52} & \textbf{30.22} \\
$\Delta$ & \color{Red} +0.33 & \color{Red} +0.67 & \color{Green} -0.01 & \color{Red} +0.34 & \color{Red} +0.29 & \color{Red} +0.02 \\

\midrule
MART & 76.77 & 68.62 & 56.90 & 51.46 & 42.07 & 31.47 \\
MART+RiFT & \textbf{77.14} & \textbf{69.41} & \textbf{56.92} & \textbf{52.42} & \textbf{43.35} & \textbf{31.48} \\
$\Delta$ & \color{Red} +0.37 & \color{Red} +0.79 & \color{Red} +0.02 & \color{Red} +0.96 & \color{Red} +1.28 & \color{Red} +0.01 \\

\midrule
AWP & 78.40 & 70.48 & 53.83 & 52.85 & 43.10 & 31.00 \\
AWP+RiFT & \textbf{78.79} & \textbf{71.12} & \textbf{53.84} & \textbf{54.89} & \textbf{45.08} & \textbf{31.05} \\
$\Delta$ &\color{Red} + 0.39 & \color{Red} +0.64 & \color{Red} +0.01  & \color{Red} +2.04 & \color{Red} +1.98 & \color{Red} +0.05 \\

\midrule
SCORE & 84.20 & 75.82 & 54.59 & 54.83 & 45.39 & 29.49 \\
SCORE+RiFT & \textbf{85.65} & \textbf{77.37} & \textbf{54.62} & \textbf{57.63} & \textbf{47.77} & \textbf{29.50} \\
$\Delta$ & \color{Red} +1.45  & \color{Red} +1.55  & \color{Red} +0.03  & \color{Red} +2.80 & \color{Red} +2.38 & \color{Red} +0.01 \\
\bottomrule
\end{tabular}
}
}
\hfill
{
\parbox{.4\linewidth}{

\caption{Results of fine-tuning on different modules.}
\label{table: fine-tuning on different modules}
\centering
\resizebox{.4\textwidth}{!}{
\begin{tabular}{ c c c c}

\toprule

Method & \textit{Std} & \textit{OOD} & \textit{Adv} \\

\midrule
All layers & 83.56 & 75.48 & 52.66 \\
Last layer & 83.35 & 75.16 & 52.75 \\
Robust-critical & 83.36 & 75.42 & 52.48 \\
Non-robust-critical & 83.44 & 75.69 & \textbf{53.65} \\
\bottomrule
\end{tabular}
}

\vfill
\vspace{30pt}

\caption{Results of fine-tuning on multiple non-robust-critical modules.}
\label{table: fine-tuning on multiple modules}
\centering
\begin{tabular}{ c c c c}

\toprule
 
Method & \textit{Std} & \textit{OOD} & \textit{Adv} \\
\midrule
Top 1 & 83.44 & 75.69 & \textbf{53.65} \\
Top 2 & 83.41 & 75.61 & 52.47 \\
Top 3 & 83.59 & 75.77 & 52.22 \\
Top 5 & 83.70 & 75.82 & 52.35 \\
\bottomrule
\end{tabular}
}
}
\end{table*}

To further validate the effectiveness of \method, we conduct experiments on ResNet18 \cite{he2016deep} trained on CIFAR10 and CIFAR100 \cite{cifar} using four different adversarial training techniques: TRADES \cite{trades}, MART \cite{mart}, AWP \cite{awp}, and SCORE \cite{score}, and then apply our \method to the resulting models. As shown in \tablename~\ref{table: combine other tricks}, our approach is compatible with various adversarial training methods and improves generalization and OOD robustness.

\subsection{Ablation Study}
\label{sec: 5.5 ablation study}
\paragraph{Fine-tuning on different modules}

To evaluate the efficacy of fine-tuning the non-robust-critical module, we conducted further experiments by fine-tuning the adversarially trained model on different modules. Specifically, we used four fine-tuning methods: fully fine-tuning, linear probing (fine-tuning on the last layer), fine-tuning on the non-robust-critical module, and fine-tuning on the robust-critical module. The experiment was conducted using ResNet18 on CIFAR-10, and the results are presented in \figurename~\ref{fig: interpolation} and \tablename~\ref{table: fine-tuning on different modules}. As described in Section \ref{section 3.1}, MRC is an upper bound for weight perturbation, indicating the criticality of a module in terms of model robustness. Fine-tuning on a non-robust-critical module can help preserve adversarial robustness but does not guarantee improvement in generalization.
Similarly, fine-tuning on the robust-critical module does not necessarily hurt robustness. However, our experiments observed that all fine-tuning methods improved generalization ability, but only fine-tuning on non-robust-critical module preserved adversarial robustness. Moreover, fine-tuning on the robust-critical module exhibited the worst trade-off between generalization and robustness compared to fine-tuning on all layers.

\paragraph{More non-robust-critical modules, more useful?}

To investigate whether fine-tuning on more non-critical modules could further improve generalization, we additionally fine-tune on the top two, top three, and top five non-robust-critical modules. However, \tablename~\ref{table: fine-tuning on different modules} reveals that generalization and OOD robustness did not surpass the results achieved by fine-tuning a singular non-robust-critical module. Notably, performance deteriorated when fine-tuning multiple non-critical modules compared to fine-tuning all layers. It's pivotal to note that this doesn't negate MRC's applicability to several modules. The MRC for module $i$ is evaluated with other module parameters held constant, making it challenging to discern the impact of worst-case perturbations across multiple modules using the MRC of a single one. We posit that broadening MRC's definition to encompass multiple modules might address this problem.

\paragraph{Ablation on interpolation factor $\alpha^*$}
The value of $\alpha^*$ is closely related to the fine-tuning learning rate. Specifically, a large learning rate can result in substantial weight updates that may push the fine-tuned weights $\vtheta_{FT}$ away from their adversarially trained counterparts $\vtheta_{AT}$. Our empirical results indicate that a fine-tuning learning rate of $0.001$ is suitable for most cases and that the corresponding $\alpha^*$ value generally ranges between $0.6$ to $0.9$.

\paragraph{Factors related to the generalization gain of \method}

"Our results unveiled patterns and behaviors that offer insights into the determinants of the generalization gains observed with RiFT. First, the generalization gain of \method is a function of both the neural network's inherent capacity and the inherent difficulty posed by the classification task. Specifically, as the classification task becomes more challenging, the robust criticality of each module increases, which in turn decreases the generalization gain of RiFT. This effect can be mitigated by using a model with a larger capacity. For instance, we observe that the generalization gain of \method increases as we switch from ResNet18 to ResNet34 and to WRN34-10 when evaluating on CIFAR100 and Tiny-ImageNet. Further, We observed that the generalization gain of RiFT with WRN34-10 on CIFAR10 is notably lower, at approximately $0.5\%$, compared to $2\%$ on other datasets. This might be attributed to the minimal generalization disparity between adversarially trained models and their standard-trained counterparts; specifically, while WRN34-10's standard test accuracy stands at around $95\%$, its adversarial counterpart registers at $87\%$. It is evident that fine-tuning on a single module may not yield significant improvements. Investigating these patterns further could offer strategies for enhancing the robustness and generalization capabilities of deep neural networks.

\section{Conclusion}
In this paper, we aim to exploit the redundant capacity of the adversarially trained models.
Our proposed \method leverages the concept of module robust criticality (MRC) to guide the fine-tuning process, which leads to improved generalization and OOD robustness. The extensive experiments demonstrate the effectiveness of RiFT across various network architectures and datasets. Our findings shed light on the intricate relationship between generalization, adversarial robustness, and OOD robustness. RiFT is a primary exploration of fine-tuning the adversarially trained models. We believe that fine-tuning holds great promise, and we call for more theoretical and empirical analyses to advance our understanding of this important technique.


{\small
\bibliographystyle{ieee_fullname}
\bibliography{egbib}
}

\clearpage
\newpage
\appendix

\renewcommand\thefigure{\thesection.\arabic{figure}}    
\setcounter{figure}{0} 

\section{Proof of the scale-invariant property}
\label{append-proof}

Without loss of generality, assume a two layers neural network $f$ and $\phi$ is a ReLU-based activation function. 
\begin{equation}
    f(\vtheta_f, \vx) = \vtheta^{(2)} \phi (\vtheta^{(1)}  \vx).
\end{equation}

The corresponding scaled neural network $g$ is:
\begin{equation}
    g(\vtheta_g, \vx) = \frac{1}{\beta} \vtheta^{(2)} \phi ( \beta \vtheta^{(1)} \vx),
\end{equation}
where the non-negative $\beta$ is the scaling factor. 

Suppose we calculate the MRC value of the first module $\vtheta^{(1)}$ and $\frac{1}{\beta} \vtheta^{(1)}$.

\begin{theorem}
\label{theorem: homo}
The rectified function $\phi(x) = \max (x, 0)$ is a homogeneous function where
\begin{align}
    \forall (z, \beta) \in \mathbb{R} \times \mathbb{R^+}, \, \phi(\beta z) = \beta \phi (z).
\end{align}
\end{theorem}
\begin{proof}
\begin{align}
    \phi(\beta z) = \max (\beta z, 0) = \beta \max (z, 0) = \beta \phi (z).
\end{align}
\end{proof}

\begin{theorem}
\label{theorem: equality}
    $\forall \vx, f(\vtheta_f, \vx) \equiv g(\vtheta_g, \vx)$.
\end{theorem}
\begin{proof}
\begin{align}
    g(\vtheta_g, \vx) &= \frac{1}{\beta} \vtheta^{(2)} \phi ( \beta \vtheta^{(1)} \vx) \\
    &\equiv  \frac{1}{\beta} \beta \vtheta^{(2)} \phi ( \vtheta^{(1)} \vx) \\
    &\equiv \vtheta^{(2)} \phi ( \vtheta^{(1)} \vx) \\
    &\equiv f(\vtheta_f, \vx)
\end{align}
\end{proof}





\begin{theorem} 
\label{theorem: robust loss euqality}
The robust losses of $f$ and $g$ are equal:
\begin{align}
    \mathcal{R}(f(\vtheta_f), \mathcal{D}) \equiv \mathcal{R}(g(\vtheta_g), \mathcal{D}).
\end{align}
\end{theorem}

\begin{proof}
According to Theorem \ref{theorem: equality},
\begin{align}
\forall \vx + \Delta \vx, f(\vtheta_f, \vx + \Delta \vx) \equiv g(\vtheta_g, \vx + \Delta \vx).
\end{align}

Thus,
\begin{align}
     & \mathop{ \max}_{\Delta \vx \in \mathcal{S}} \ell (f ( \vtheta_f, \vx + \Delta \vx), y) \\ 
    \equiv & \mathop{ \max}_{\Delta \vx \in \mathcal{S}} \ell (g ( \vtheta_g, \vx + \Delta \vx), y).
\end{align}

Thus,
\begin{align}
\mathcal{R}(f(\vtheta_f), \mathcal{D}) &= \sum \limits_{(\vx, y) \in \mathcal{D}} \max_{\Delta \vx \in \mathcal{S}} \ell (f ( \vtheta_f, \vx + \Delta \vx), y) \\
&\equiv \sum \limits_{(\vx, y) \in \mathcal{D}} \max_{\Delta \vx \in \mathcal{S}} \ell (g ( \vtheta_g, \vx + \Delta \vx), y) \\
&= \mathcal{R}(g(\vtheta_g), \mathcal{D})
\end{align}

\end{proof}

\begin{theorem}
The Module Robustness Criticality (MRC) proposed in Definition \ref{def: mrc} is invariant to the scaling of the parameters.
\end{theorem}
\begin{proof}

Let $\Delta \vtheta_f = \{ \Delta \vtheta_f^{(1)}, \mathbf{0} \}, \Delta \vtheta_g = \{ \Delta \vtheta_g^{(1)}, \mathbf{0} \}$ be the perturbation of the first layer for network $f$ and $g$ respectively. First, we prove 
\begin{align}
&\max_{\Delta \vtheta_f \in \mathcal{C}_{\vtheta_f}} \mathcal{R}(f(\vtheta_f + \Delta \vtheta_f), \mathcal{D}) \\ \leq &\max_{\Delta \vtheta_g \in \mathcal{C}_{\vtheta_g}} \mathcal{R}( g(\vtheta_g + \Delta \vtheta_g), \mathcal{D}).
\end{align}

Let
\begin{align}
    \Delta \vtheta_f^* = \mathop{\arg \max} \limits_{\Delta \vtheta_f \in \mathcal{C}_{\theta_f}} \mathcal{R}(f(\vtheta_f + \Delta \vtheta_f), \mathcal{D}), \\
    \Delta \vtheta_g^* = \mathop{\arg \max} \limits_{\Delta \vtheta_g \in \mathcal{C}_{\theta_g}} \mathcal{R}(g(\vtheta_g + \Delta \vtheta_g), \mathcal{D}).
\end{align}

Consider the perturbation $\Delta \Tilde{ \vtheta_g } = \beta \Delta \vtheta_f^*$ for $g$, it is easy to show that $\Delta \Tilde{ \vtheta_g }\in \mathcal{C}_{\theta_g}$,
\begin{align}
    \mathcal{C}_{\vtheta_f} &= \{ \Delta \vtheta_f \, \big|  \, \lVert \Delta \vtheta_f \rVert_p \leq \epsilon \lVert \vtheta_f^{(1)} \rVert_p \},\\
    \mathcal{C}_{\vtheta_g} &= \{ \Delta \vtheta_g \, \big|  \,  \lVert \Delta \vtheta_g \rVert_p \leq \epsilon \lVert \vtheta_g^{(1)} \rVert_p \} \\
    &=\{ \Delta \vtheta_g \, \big|  \, \Delta \vtheta_g = \beta \lVert \Delta \vtheta_f \rVert_p \leq \epsilon \beta\lVert \vtheta_f^{(1)} \rVert_p \}.
\end{align}

Therefore,
{\small
\begin{align}
     \mathcal{R}(g(\vtheta_g + \Delta \Tilde{ \vtheta_g }), \mathcal{D}) &\leq \mathop{ \max} \limits_{\Delta \vtheta_g \in \mathcal{C}_{\theta_g}} \mathcal{R}(g(\vtheta_g + \Delta \vtheta_g), \mathcal{D})\\
     &= \mathcal{R}(g(\vtheta_g + \Delta \vtheta_g^*), \mathcal{D}).
\end{align}
}

Repeat the same analysis as presented in Theorem \ref{theorem: equality},
\begin{align}
    &g(\vtheta_g + \Delta \Tilde{ \vtheta_g }) \\
    = &\frac{1}{\beta} \vtheta^{(2)} \phi ( (\beta \vtheta^{(1)} + \beta \Delta \vtheta_f^*) \vx) \\
    = &\vtheta^{(2)} \phi ( ( \vtheta^{(1)} + \Delta \vtheta_f^*) \vx) \\
    \equiv &f(\vtheta_f + \Delta \vtheta_f^*).
\end{align}

According to Theorem \ref{theorem: robust loss euqality},
\begin{align}
     \mathcal{R}(f(\vtheta_f + \Delta \vtheta_f^*), \mathcal{D})  &\equiv \mathcal{R}(g(\vtheta_g + \Delta \Tilde{ \vtheta_g }), \mathcal{D}) \\
     &\leq \mathcal{R}(g(\vtheta_g + \Delta \vtheta_g^*), \mathcal{D}).
\end{align}

Similarly, we can prove 
\begin{align}
& \max_{\Delta \vtheta_g \in \mathcal{C}_{\vtheta_g}} \mathcal{R}( g(\vtheta_g + \Delta \vtheta_g), \mathcal{D}) \\ \leq & \max_{\Delta \vtheta_f \in \mathcal{C}_{\vtheta_f}} \mathcal{R}(f(\vtheta_f + \Delta \vtheta_f), \mathcal{D}). 
\end{align}

Thus, 
\begin{align}
&\max_{\Delta \vtheta_f \in \mathcal{C}_{\theta_f}} \mathcal{R}(f(\theta_f + \Delta \vtheta_f), \mathcal{D}) \\
= &\max_{\Delta \vtheta_g \in C_{\vtheta_g}} \mathcal{R} (g(\theta_g + \Delta \vtheta_g), \mathcal{D}).
\end{align}

Such that the proof ends.
\end{proof}

\section{Algorithm of \method}
\label{append-algo}

The complete algorithm of \method is presented in Algorithm \ref{alg: rift}.

\begin{algorithm}[htbp]
\caption{Robust Critical Fine-Tuning}
\label{alg: rift}
\begin{algorithmic}[1]
\Require adversarially trained model weights $\vtheta_{AT}$, standard dataset $\mathcal{D}_{std}$, weight perturbation scaling factor $\alpha$, fine-tuning optimization iteration steps $T$ and learning rate $\gamma$, weight decay facotr $\lambda$.
\Ensure The fine-tuned model weights $\vtheta_{AT}^*$.
\State{\textbf{Step 1}: Calculate MRC for each module}

\For{Module weight $\vtheta^{(j)}$}
\State{Calculate MRC value of $\vtheta^{(j)}$ using Algorithm \ref{alg: mrc calculation}.}
\EndFor

\State{Select the module with lowest MRC value, denote as non-robust critical module $\vtheta^{(i)}$}

\State{\textbf{Step 2}: Fine-tuning on Non-robust critical module}
\State{$\vtheta_{1} = \vtheta_{AT}$}
\For{$t = 1, \hdots, T$} \Comment{Fine-tuning $T$ epochs}
\For{Batch $\mathcal{B}_k \in \mathcal{D}_{std}$}
\State{Calculate loss: $\mathcal{L}(f(\vtheta_t), \mathcal{B}_k))$}
\State{$\vtheta^{(i)}_{t+1} = \vtheta^{(i)}_{t+1} - \gamma \nabla_{\vtheta_t}(\mathcal{L})$} 
\Comment{Gradient Descent}
\EndFor
\State{$\vtheta_{FT} = \vtheta_t$ if $\vtheta_t$ obtain highest std test acc.}
\EndFor

\State{\textbf{Step 3}: Interpolation}
\For{$\alpha \in (0, 1, 0.05)$}
\State{$\vtheta_\alpha = (1-\alpha) \vtheta_{AT} + \alpha \vtheta_{FT}$}
\State{$\vtheta_{FT}^* = \vtheta_\alpha$ if it reaches best standard test acc while preserve the robustness as $\vtheta_{AT}$.}
\EndFor

\State{\textbf{Return} Fine-tuned model weights $\vtheta_{FT}^*$}
\end{algorithmic}
\end{algorithm}

\section{Training Details}
\label{append-training details}

\subsection{Experiment Environment}
All experiments are conducted on a workstation equipped with an NVIDIA GeForce RTX 3090 GPU with 24GB memory and NVIDIA A100 with 80GB memory. The PyTorch version is 1.11.0.

\subsection{Adversarial Training Details}
For vanilla adversarial training, We set the initial learning rate as 0.1, which decays at 100 and 105 epochs with factor 10. When generating adversarial examples, we set BN as train mode since it usually achieves higher robustness.

When incorporating RiFT with other adversarial training methods, the SCORE method is incorporated with TRADES. For the CIFAR100 training, we ran with three different learning rate and select the best model weights as the one with highest robust accuracy. The hyper-parameter settings are either based on their original paper or same as the vanilla AT, depends on which method achieves better robust accuracy.

\subsection{Fine-tuning Details}
The hyper-parameter that most affects fine-tuning is the initial learning rate. According to our experience, we find a small learning rate usually performs better. If the adversarial robustness of the final fine-tuned weights is still higher than the robustness of the initial adversarial training, we then increase the learning rate.

\subsection{The MRC value of ResNet34 and WRN34-10}
\label{append-mrc value plots}
\figurename~\ref{append fig: R34 module robust criticality} and \figurename~\ref{append fig: WRN34 module robust criticality} shows the Module Robust Criticality (MRC) value of each module in ResNet34 trained on CIFAR100 and WideResNet34 trained on Tiny-ImageNet, respectively. It can be observed that both models exhibit redundant capacity. Additionally, \figurename~\ref{append fig: R18 cifar100 module robust criticality} and \figurename~\ref{append fig: R18 tiny module robust criticality} shows the MRC value of each module in ResNet18 trained on CIFAR100 and Tiny-ImageNet, respectively. As we discussed in Section~\ref{sec: 5.3 main results} and Section~\ref{sec: 5.5 ablation study}, ResNet18 has a lower redundant capacity compared to ResNet34 and WideResNet34, and the redundant capacity decreases as the classification task becomes more complex.

\begin{figure*}[ht]
\centering
\includegraphics[width=\textwidth]{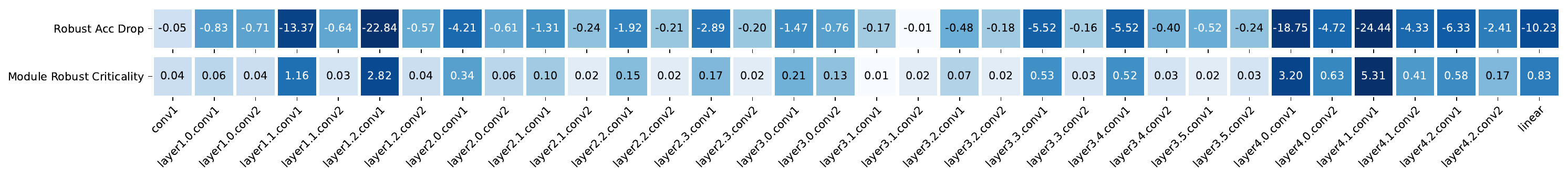}
\caption{Example of module robust criticality (MRC) and its corresponding robust accuracy drop of ResNet34 trained on CIFAR100.}
\label{append fig: R34 module robust criticality}
\end{figure*}

\begin{figure*}[ht]
\centering
\includegraphics[width=\textwidth]{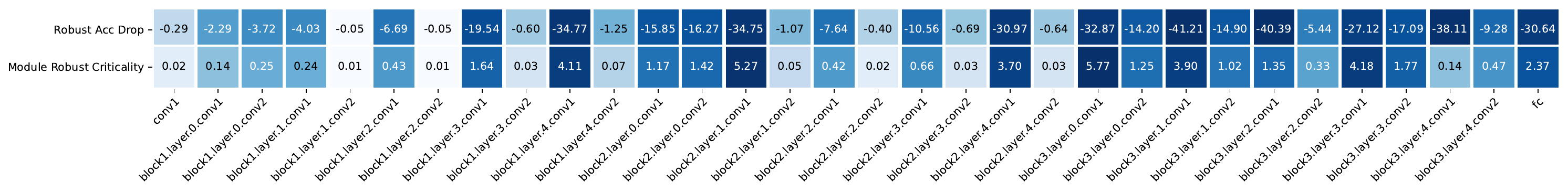}
\caption{Example of module robust criticality (MRC) and its corresponding robust accuracy drop of WideResNet34 trained on Tiny-ImageNet.}
\label{append fig: WRN34 module robust criticality}
\end{figure*}

\begin{figure*}[ht]
\centering
\includegraphics[width=\textwidth]{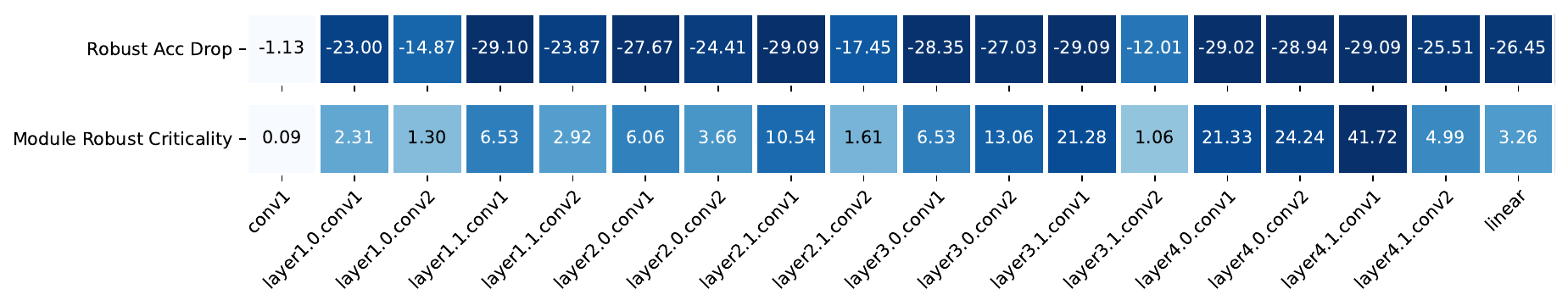}
\caption{Example of module robust criticality (MRC) and its corresponding robust accuracy drop of ResNet18 trained on CIFAR100.}
\label{append fig: R18 cifar100 module robust criticality}
\end{figure*}

\begin{figure*}[ht]
\centering
\includegraphics[width=\textwidth]{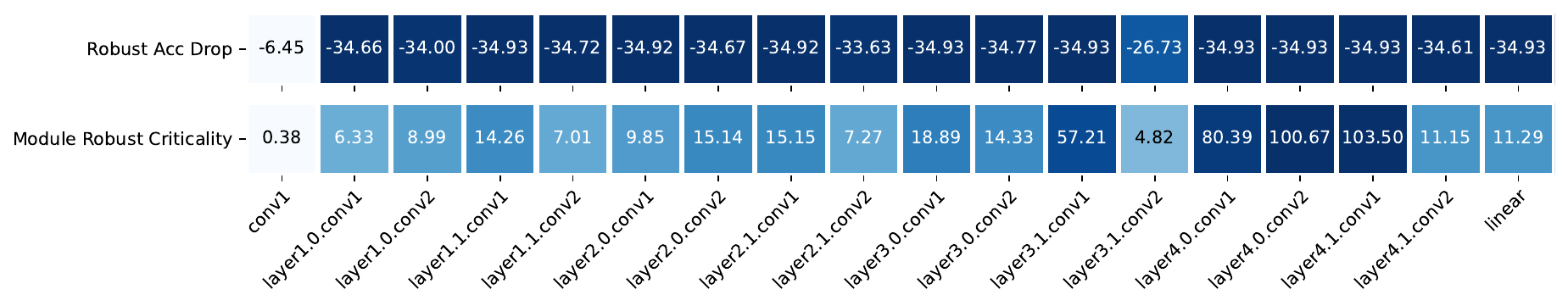}
\caption{Example of module robust criticality (MRC) and its corresponding robust accuracy drop of ResNet18 trained on Tiny-ImageNet.}
\label{append fig: R18 tiny module robust criticality}
\end{figure*}

\subsection{More interpolation results}
\label{append-interpolation}

\figurename~\ref{append fig: extra interpolation} shows the interpolation results of different modules of ResNet18 trained on CIFAR100 dataset. It can be observed that fine-tuning on robust-critical module can also help improve generalization and robustness. This does not mean that our MRC is wrong, as we claimed in Section~\ref{sec 5.2: mrc analysis}, fine-tuning on robust-critical module does not necessarily hurt robustness. The MRC provides guidance on which module to fine-tune for optimal results, and still, fine-tuning on non-robust-critical module achieves the highest test accuracy while preserving robustness.

\begin{figure}[ht]
\centering
\includegraphics[width=0.5\textwidth]{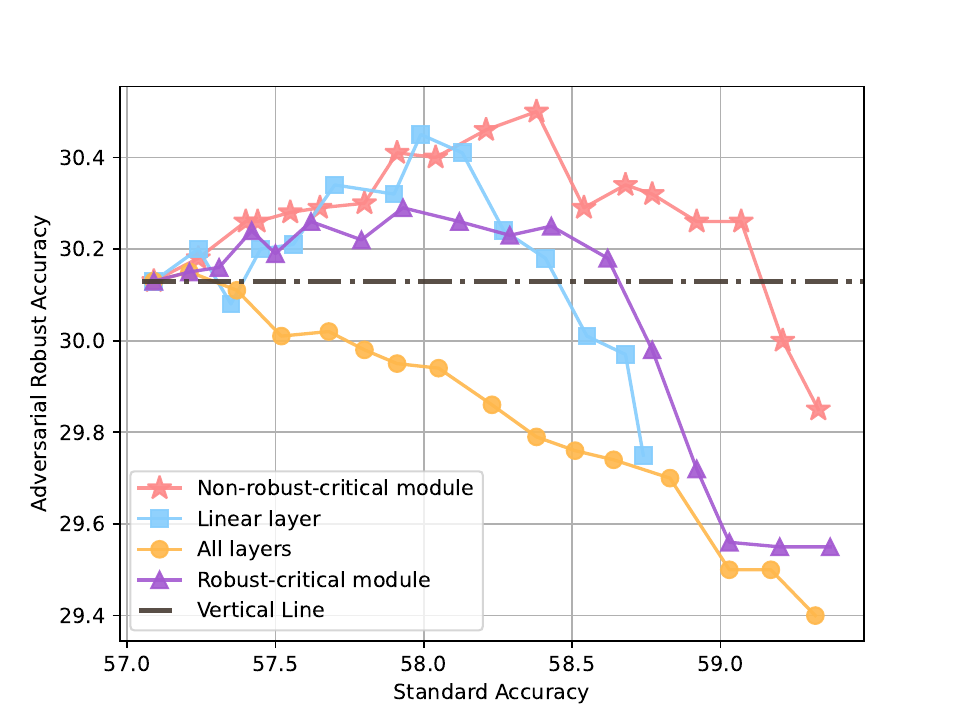}
\caption{Interpolation results of fine-tuning on different modules of ResNet18 on CIFAR100 dataset. 
Dots denote different interpolation points between the final fine-tuned weights of \method and the initial adversarially trained weights.}
\label{append fig: extra interpolation}
\end{figure}

\section{Analysis of the complexity of MRC algorithm}
When identifying the most non-robust-critical module, it is required to iterate all modules of the model. Suppose a model with $n$ modules, for each module, the calculation complexity depends on the iteration steps in Algorithm \ref{alg: mrc calculation}. Considering the different overheads for each iterative computation of the modules at different locations, for example, when calculating the last module's MRC value, it only requires forward-backward iteration of the last layer of parameters. Thus, the average total forward-backward iteration of each module is $n/2$. In our experiments, we set the learning rate as 1 and the iteration step as 10. Thus, in our experiments, the complexity of MRC algorithm cost $5n$ total forward-backward propagation.

\end{document}